\pdfoutput=1 
\documentclass[11pt]{article}

\usepackage[font=libertinus, citestyle=authoryear]{kurbanlab}

\DeclareAffiliation{hbku}{%
  College of Science and Engineering, Hamad Bin Khalifa University, Doha, Qatar}

\DeclareAffiliation{tamu}{%
  Department of Computer and Electrical Engineering,
  Texas A\&M University, College Station, TX, USA}

\DeclareAffiliation{iub}{%
  Luddy School of Informatics, Computing, and Engineering,
  Indiana University Bloomington, Bloomington, IN, USA}

\usetikzlibrary{arrows.meta,positioning,calc,backgrounds}
\usepackage{colortbl}                       
\definecolor{cx}{HTML}{2166AC}              
\definecolor{cxt}{HTML}{D6604D}
\definecolor{chl}{HTML}{1B7837}
\newcommand{\Om}{\ensuremath{\Omega}}
\newcommand{\Gcov}{\ensuremath{\Gamma_{\mathrm{cov}}}}
\newcommand{\Goov}{\ensuremath{\Gamma_{\mathrm{oov}}}}
\newcommand{\Vm}{\ensuremath{V}}

\title{The Spectrum Is Not Enough}
\Subtitle{When Context Helps Time-Series Forecasting}
\RunningTitle{The Spectrum Is Not Enough}

\Author{Mert Onur Cakiroglu}{iub}
\Author{Mehmet Dalkilic}{iub}
\Author[corresponding=hkurban@hbku.edu.qa, orcid=0000-0003-3142-2866]{Hasan Kurban}{hbku}

\CodeURL{https://github.com/KurbanIntelligenceLab/SINE}

\begin{document}
\maketitle

\begin{abstract}
A growing family of indices scores how predictable a series is from its spectrum. Practitioners increasingly read these scores as answering a different question: whether \emph{adding context}, a longer lookback, a retrieval plug-in, or a pretrained model, will help. These are not the same question. The value of context is a property of the operating point, not of the series. Any index built from the power spectrum is invariant under phase randomization, whereas the beyond-second-order value that retrieval and foundation models supply is not, because a phase-randomized series is asymptotically Gaussian. We state this as an impossibility result and isolate it with surrogate pairs that fix the spectrum and the marginal by construction. We then give a label-free, configuration-level diagnostic, the coverage deficit, whose principal term measures beyond-spectrum structure as the gain of analog over linear prediction. On seven benchmarks the prediction holds: window-keyed retrieval's value collapses across surrogate pairs (ECL median $+33\%\!\to\!-35\%$, $p{<}10^{-40}$) while every spectral index stays frozen; a foundation model's value splits into a surviving second-order part and a small beyond-linear margin that collapses; a longer linear window's value survives. Leave-one-dataset-out, the structure term predicts the sign of beyond-spectrum value where the spectral indices trail it, and the reverse holds for the second-order mechanism. We introduce no new forecaster; the contribution is the distinction, a controlled comparison, and a diagnostic for the deployment decision. Code: \url{https://github.com/KurbanIntelligenceLab/SINE}.
\end{abstract}

\printkeywords

\section{Introduction}
\label{sec:intro}
Two series can be equally predictable yet differ in whether a longer history improves forecasts. This distinction is at odds with how a growing family of predictability measures is being applied.

Recent work scores the \emph{predictability of a series} with a single inexpensive scalar: spectral predictability, which reports that large pretrained models outperform light baselines when it is high \citep{omega2026}; a minimum achievable error from second-order structure with a computable spectral surrogate \citep{scp2025}; an accuracy law relating window-wise complexity to the smallest error deep models reach \citep{accuracylaw2025}; and conditional-entropy forecastability profiles over horizons \citep{catt2026profile}. These address a well-posed question: how predictable a series is in principle.

\begin{figure}[t]
\centering
\resizebox{\textwidth}{!}{%
\begin{tikzpicture}[
  font=\sffamily\footnotesize,
  >={Latex[length=2.0mm]},
  arr/.style={->,line width=0.8pt,black!45},
  wpanel/.style={rounded corners=3pt,line width=0.9pt},
  base/.style={black!16,line width=0.55pt},
  plab/.style={font=\sffamily\bfseries\small,text=black!65}]

\node[plab] at (0.15,3.70){a};
\begin{scope}[shift={(0.36,1.42)}]
  \draw[base] (-0.05,0)--(2.22,0);
  \foreach \i/\h/\op in {0/0.26/40,1/0.68/54,2/1.06/70,3/0.44/50,4/0.22/42,5/0.14/36,6/0.10/32}
    \fill[black!\op] ({\i*0.305},0) rectangle ({\i*0.305+0.185},\h);
\end{scope}
\node[font=\sffamily\scriptsize\bfseries,text=black!68] at (1.45,1.04){shared power spectrum};
\node[font=\sffamily\scriptsize,text=black!50] at (1.45,0.70){$P(\textcolor{cx}{x})=P(\textcolor{cxt}{\tilde{x}})$};

\node[plab] at (4.30,3.70){b};
\node[wpanel,draw=cx!45,fill=cx!6,minimum width=3.5cm,minimum height=1.04cm](wx) at (6.15,3.06){};
\node[font=\sffamily\scriptsize\bfseries,text=cx,anchor=west] at (4.55,3.70){$x$\,: phase-coherent};
\draw[base] (4.55,3.06)--(7.75,3.06);
\begin{scope}[shift={(4.62,3.06)}]
  \draw[cx,line width=1.3pt,domain=0:6.283,samples=170,smooth]
    plot ({\x*0.475},{0.30*sin(\x r)+0.10*sin(2*\x r+0.5)});
\end{scope}

\node[plab] at (4.30,1.92){c};
\node[wpanel,draw=cxt!45,fill=cxt!6,minimum width=3.5cm,minimum height=1.04cm](wxt) at (6.15,1.26){};
\node[font=\sffamily\scriptsize\bfseries,text=cxt,anchor=west] at (4.55,1.92){$\tilde{x}$\,: phase-randomized};
\draw[base] (4.55,1.26)--(7.75,1.26);
\begin{scope}[shift={(4.62,1.26)}]
  \draw[cxt,line width=1.3pt,domain=0:6.283,samples=200,smooth]
    plot ({\x*0.475},{0.24*sin(2.3*\x r+1.7)-0.10*sin(5.1*\x r)+0.07*sin(3.3*\x r+0.6)});
\end{scope}

\draw[arr,line width=1.0pt] (7.40,2.54) -- (7.40,1.78);
\node[draw=black!28,rounded corners=2.5pt,fill=white,inner sep=1.8pt,
      font=\sffamily\scriptsize,text=black!68] at (7.40,2.28){phase randomization};

\draw[arr] (2.62,2.13) to[out=24,in=183] (4.32,3.06);
\draw[arr] (2.62,2.13) to[out=-24,in=177] (4.32,1.26);

\node[plab] at (8.62,3.70){d};
\node[font=\sffamily\scriptsize\bfseries,text=black!68,anchor=west] at (8.86,3.70){value of context $\Vm_M$};
\draw[arr] (7.98,3.06)--(8.74,3.06);
\draw[arr] (7.98,1.26)--(8.74,1.26);
\draw[rounded corners=2pt,draw=black!25,line width=0.7pt] (8.80,2.86) rectangle (11.80,3.26);
\draw[black!45,line width=0.6pt,densely dotted] (10.30,2.82)--(10.30,3.30);
\fill[cx!85,rounded corners=1pt] (10.32,2.90) rectangle (11.62,3.22);
\node[font=\sffamily\bfseries\scriptsize,text=white,anchor=east] at (11.56,3.06){$+33\%$};
\node[font=\sffamily\scriptsize,text=cx,anchor=west] at (11.94,3.06){helps};
\draw[rounded corners=2pt,draw=black!25,line width=0.7pt] (8.80,1.06) rectangle (11.80,1.46);
\draw[black!45,line width=0.6pt,densely dotted] (10.30,1.02)--(10.30,1.50);
\fill[cxt!85,rounded corners=1pt] (8.98,1.10) rectangle (10.28,1.42);
\node[font=\sffamily\bfseries\scriptsize,text=white,anchor=west] at (9.04,1.26){$-35\%$};
\node[font=\sffamily\scriptsize,text=cxt,anchor=west] at (11.94,1.26){hurts};
\node[font=\sffamily\scriptsize,text=black!50] at (10.30,3.48){$0$};
\node[font=\sffamily\bfseries\small,text=black!55] at (10.30,2.16){$\neq$};
\node[font=\sffamily\scriptsize,text=black!52,anchor=west] at (10.55,2.16){opposite};

\node[plab] at (13.42,3.70){e};
\node[font=\sffamily\scriptsize\bfseries,text=black!68,anchor=west] at (13.66,3.70){measured on ECL};
\fill[cx] (17.10,3.70) circle (1.6pt);
\node[font=\sffamily\scriptsize,text=black!55,anchor=west] at (17.20,3.70){$x$};
\draw[cxt,line width=1.0pt] (17.62,3.70) circle (1.6pt);
\node[font=\sffamily\scriptsize,text=black!55,anchor=west] at (17.72,3.70){$\tilde{x}$};
\draw[black!45,line width=0.6pt,densely dotted] (16.96,0.62)--(16.96,3.30);
\node[font=\sffamily\scriptsize,text=black!50] at (16.96,3.44){$0$};
\foreach \y in {2.78,1.86,0.94} \draw[black!10,line width=0.5pt] (16.15,\y)--(18.05,\y);
\node[font=\sffamily\scriptsize\bfseries,text=black!62,anchor=west] at (13.55,3.10){window};
\node[font=\sffamily\scriptsize,text=black!48,anchor=east,fill=white,inner sep=1pt] at (18.05,3.10){$56.7\%\!\to\!55.1\%$};
\draw[black!35,line width=0.8pt] (17.981,2.78)--(17.952,2.78);
\fill[cx] (17.981,2.78) circle (1.9pt);
\draw[cxt,line width=1.0pt] (17.952,2.78) circle (1.9pt);
\node[font=\sffamily\scriptsize,text=black!55,anchor=east] at (18.05,2.50){survives};
\node[font=\sffamily\scriptsize\bfseries,text=black!62,anchor=west] at (13.55,2.18){retrieval};
\node[font=\sffamily\scriptsize,text=black!48,anchor=east,fill=white,inner sep=1pt] at (18.05,2.18){$+33.0\%\!\to\!-35.0\%$};
\draw[->,black!35,line width=0.8pt,shorten >=2.4pt,shorten <=2.4pt] (17.554,1.86)--(16.330,1.86);
\fill[cx] (17.554,1.86) circle (1.9pt);
\draw[cxt,line width=1.0pt] (16.330,1.86) circle (1.9pt);
\node[font=\sffamily\scriptsize,text=cxt,anchor=east] at (16.88,1.58){collapses};
\node[font=\sffamily\scriptsize\bfseries,text=black!62,anchor=west] at (13.55,1.26){FM margin};
\node[font=\sffamily\scriptsize,text=black!48,anchor=east,fill=white,inner sep=1pt] at (18.05,1.26){$+9.0\%\!\to\!-8.8\%$};
\draw[->,black!35,line width=0.8pt,shorten >=2.4pt,shorten <=2.4pt] (17.122,0.94)--(16.802,0.94);
\fill[cx] (17.122,0.94) circle (1.9pt);
\draw[cxt,line width=1.0pt] (16.802,0.94) circle (1.9pt);
\end{tikzpicture}}
\caption{\textbf{Identical predictability, opposite value of context.} \textbf{(a--c)}~A series $x$ and its phase-randomized surrogate $\tilde{x}$ share the power spectrum, and after amplitude adjustment the marginal, so every power-spectrum index scores both identically. \textbf{(d)}~The beyond-second-order value of context $\Vm_M$ is nonetheless large on $x$ and negative on $\tilde{x}$. \textbf{(e)}~The measured dissociation on ECL (channel medians; Tables~\ref{tab:gap} and~\ref{tab:fmmargin}): the longer window's purely second-order gain survives phase randomization, while retrieval and the foundation model's beyond-linear margin collapse through zero, the structure the spectrum cannot see.}
\label{fig:teaser}
\end{figure}
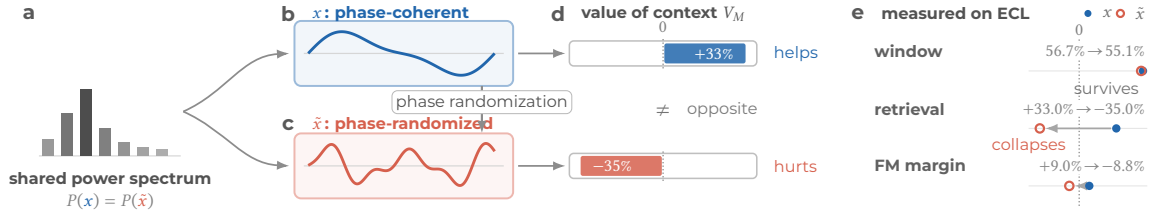

Practitioners face a different question. Given a series and a deployment configuration, will \emph{adding context} pay off: a longer lookback, a retrieval plug-in that fetches from the training record \citep{han2025raft}, or a foundation model that carries broad temporal priors \citep{ansari2024chronos,woo2024moirai}? A high predictability score is easily read as license for the heavier option and a low one as a reason to stay simple, and the stakes are real: large-scale re-evaluations report that supervised long-term forecasting rankings flip under small changes of setup or metric \citep{brigato2026champions}. Determining when additional context will improve forecasting remains an active deployment problem. Recent work addresses different aspects of this decision, including selecting the appropriate lookback for each task \citep{abdelmalak2026}, explaining retrospectively when foundation models perform well \citep{whenfm2026}, and redesigning retrieval methods to capture phase-dependent structure that conventional spectral representations overlook \citep{specretf2026}. We show that this reading is unsound: the benefit of context is not a property of the series, and we identify the property that governs it.

Predictability and the benefit of added context diverge when predictability indices cannot distinguish cases in which added context yields different benefits. This limitation follows from what these indices measure. The power spectrum records how much energy sits at each frequency, but it does not preserve the phase structure needed to identify the series' current position within those cycles or determine whether the same position leads to a repeatable future pattern. Phase is precisely what a short window may fail to carry and additional context can supply. A window spanning a full dominant period contains the complete cycle, whereas a shorter one may not, and no within-window model can recover information that is absent \citep{butera2026}. Whether supplying that missing phase is worthwhile depends on whether it recurs across cycles, beyond the spectrum's reach.

We formalize this using phase-randomized surrogate series. Phase randomization \citep{theiler1992,schreiber2000} transforms any series into a surrogate with the same power spectrum but randomized Fourier phases, while an amplitude-adjusted variant also preserves the marginal. Any power-spectrum index is therefore identical on a series and its surrogate. Yet the surrogate is asymptotically Gaussian, so the beyond-second-order value exploited by retrieval and foundation models collapses on it. The spectrum therefore identifies what these indices preserve, but not the structure that determines whether additional context is useful. Figure~\ref{fig:teaser} summarizes the consequence.

We therefore introduce the coverage deficit, a configuration-level diagnostic computed before deployment without test labels. Its principal term combines a measure of beyond-spectrum structure, the gain of analog over linear prediction, with the fraction of the state-identifying motif that the window fails to observe. A second term flags distributional novelty, the regime in which any memory is stale. Across three context-extending mechanisms on seven standard benchmarks, these two terms separate the deployment question along exactly the theoretical line (Table~\ref{tab:lodo}).

We make three contributions.
\begin{itemize}
\item \textbf{Separating predictability from context value.} We distinguish series-level predictability from configuration-level context value and prove an impossibility result. Any index built from the power spectrum, which covers spectral predictability ($\Om$) \citep{omega2026,goerg2013foreca} and spectral-coherence predictability (SCP) \citep{scp2025}, is invariant under phase randomization. Because beyond-spectrum context value is not invariant, no such index can predict it. Amplitude-adjusted surrogates extend the control to indices with a distributional term such as accuracy-law complexity \citep{accuracylaw2025} (Section~\ref{sec:theory}).
\item \textbf{A spectrum-controlled comparison.} Surrogate pairs hold the spectrum \emph{and} the marginal fixed by construction while a longer window, a retrieval plug-in, and a foundation model are switched on and off. The construction fixes exactly what the competing indices read, so the comparison isolates their blind spot rather than relying on correlation (Sections~\ref{sec:theory}, \ref{sec:experiments}).
\item \textbf{A configuration-level diagnostic.} The coverage deficit is label-free and computed before deployment, and its principal term repurposes the nonlinear-prediction statistic of \citet{sugihara1990} to measure the beyond-spectrum structure the power spectrum cannot represent. Leave-one-dataset-out, it predicts the sign of beyond-spectrum context value where $\Om$ is at or below chance; SCP can exceed chance but trails it by 12--19 points (Section~\ref{sec:method}, Table~\ref{tab:lodo}).
\end{itemize}

\paragraph{Problem setup and notation.}
A forecaster maps a lookback window $x_{t-S+1:t}\in\mathbb{R}^{S\times D}$ of a $D$-channel time series to the next $H$ steps. We write $S$ for the lookback length, $H$ for the prediction horizon, and $L$ for the dominant period, estimated per channel as the peak of the training-split periodogram. A context-extending mechanism $M$ enlarges the information available to a base forecaster without changing the prediction target, for example by increasing the lookback ($S\!\rightarrow\!S'>S$), retrieving similar windows from the training record, or supplying pretrained temporal knowledge. Let $f$ denote the base forecaster and $f\oplus M$ the same forecaster augmented with $M$. We define the context value of $M$ at operating point $(S,H)$ on series $x$ as the paired relative reduction in test mean squared error (MSE),
\begin{equation}
V_M(x;S,H)=
\frac{\mathrm{MSE}(f)-\mathrm{MSE}(f\oplus M)}
{\mathrm{MSE}(f)},
\end{equation}
which is positive when $M$ improves prediction. Finally, a series-level predictability index $P(x)$ is any statistic intended to characterize the intrinsic predictability of $x$ independent of a particular operating point. Throughout, we distinguish these two quantities: $P(x)$ characterizes the series itself, whereas $V_M$ depends on both the series and the deployment configuration.

\section{Related Work}
\label{sec:related}

\paragraph{Series-level predictability indices.} A growing line scores how predictable a series is. Spectral predictability traces to forecastable-component analysis \citep{goerg2013foreca}; \citet{omega2026} revive it as $\Om$ and show, across 51 models and 28 datasets, that foundation models beat light baselines when $\Om$ is high. \citet{scp2025} derive a per-instance \emph{linear} MSE lower bound from spectral coherence. \citet{accuracylaw2025} relate a window-wise complexity to the smallest error deep models attain. The information-theoretic and dynamical route runs from model-free quantification with weighted permutation entropy \citep{garland2014} and largest-Lyapunov measures \citep{wang2025forecastability} to horizon-resolved forecastability profiles conditioned on a declared information set \citep{catt2026profile}. That profile bounds the \emph{total} improvement over the unconditional predictor, whereas our gap $\Delta$ isolates the component beyond the best linear predictor on the same access, which no power spectrum represents. We do not dispute these limits, but prove the power-spectrum ones cannot answer the deployment question they are increasingly used for. Entropy and higher-order scores fall outside the impossibility yet stay series- or access-level, and E2 tests them head to head. Our $\Delta_{\mathrm{nl}}$ is the configuration-level analogue.

\paragraph{When does a heavier option help?} The deployment question is now studied directly. \citet{abdelmalak2026} show a mis-specified lookback inverts rankings and tune it by search. \citet{butera2026} attribute long-context benefit to generative-process identification and prove a window must strictly exceed a process memory to reach the minimum error. Our spectral/beyond-spectral split refines that benefit: its second-order part is spectrum-visible and survives phase randomization (our longer-linear-window mechanism), the remainder is not. \citet{whenfm2026} rate foundation models post hoc. Symbolic memories make the operating point concrete: a de Bruijn graph over the discretized training record recovers cross-window structure at windows as short as $S{=}12$, handling at test time exactly the out-of-vocabulary event our novelty term measures \citep{cakiroglu2025dragon}. Dynamical-systems forecasters revive delay-coordinate embedding \citep{deepedm2025,attraos2024}, exploiting the beyond-spectrum structure our result concerns. None provides a label-free, pre-deployment statistic paired with a statement of what no spectral index can do.

\paragraph{Retrieval and pretraining as context.} Retrieval plug-ins inject cross-window structure by frequency statistics \citep{ye2024fan}, learned cycle embeddings \citep{lin2024cyclenet}, corpus lookup \citep{han2025raft,tire2024raf}, diffusion guidance \citep{liu2024radiff}, or per-channel retrieval \citep{craft2026}. Stationarity-aware variants adapt retrieval under non-stationarity \citep{saraf2026}, and long-context comparisons place retrieval against very long windows \citep{ahuja2026}. A recent redesign carries amplitude and phase in the retrieval similarity metric \citep{specretf2026}, independently pointing to phase as the relevant axis. Foundation models are the pretraining route to the same end \citep{ansari2024chronos,woo2024moirai}. Context parroting shows copying from a long context can beat them \citep{zhang2026parroting}, and their failures track spectral shift \citep{spectralshift2025}. We treat all of these as context-extending mechanisms and ask a single question across them.

\paragraph{Surrogate data and selection.} Phase-randomized and amplitude-adjusted surrogates are the classical instrument for separating linear from nonlinear structure, including assessing the significance of a nonlinear prediction gain \citep{theiler1992,schreiber2000}. We repurpose them, with the nonlinear-prediction statistic of \citet{sugihara1990} as $\Delta_{\mathrm{nl}}$, to control the exact quantities the predictability indices read; E2 adds generic catch22 features \citep{lubba2019catch22} as a selection baseline.

\section{Predictability Does Not Determine Context Value}
\label{sec:theory}

We now prove that no power-spectrum index can predict the value of beyond-spectrum context. The argument turns on the gap between what a linear predictor and the best possible predictor achieve, which the spectrum cannot see and a surrogate erases. Throughout, $x$ is a real, second-order-stationary, finite-variance series; the classical steps and all regularity conditions are deferred to the appendix.

\paragraph{Two error floors.}
Fix a horizon and an information set $\mathcal{I}$ available to a forecaster at forecast origin $t$: a length-$S$ window, optionally augmented by retrieved context or information supplied by a pretrained model. Throughout this section, let $h$ denote a predictor based on the information set $\mathcal{I}$. The minimum mean-squared error achievable by any \emph{linear} predictor is
\begin{equation}
\label{eq:sigmalin}
\sigma^2_{\mathrm{lin}}(\mathcal{I})
=
\min_{h\in\mathcal{H}_{\mathrm{lin}}}
\mathbb{E}
\left\|
x_{t+1:t+H}
-
h(\mathcal{I})
\right\|^2,
\end{equation}
where $\mathcal{H}_{\mathrm{lin}}$ denotes the class of all linear predictors based on $\mathcal{I}$. Let $\sigma^2_{*}(\mathcal{I})$ denote the Bayes error, i.e., the minimum mean-squared error over all measurable predictors based on $\mathcal{I}$. The quantity $\sigma^2_{\mathrm{lin}}(\mathcal{I})$ depends on $x$ only through its autocovariance (App.~A.2). Their difference,
\begin{equation}
\Delta(\mathcal{I})=\sigma^2_{\mathrm{lin}}(\mathcal{I})-\sigma^2_{*}(\mathcal{I})\ \ge\ 0,
\end{equation}
is the component of predictability beyond second order. Call a mechanism $M$ \emph{beyond-spectrum} if the predictability it exploits lies past second order, as analog and similarity retrieval and foundation models do and a longer \emph{linear} window does not. The following bound is the theoretical core of the paper: it ties the value of any such mechanism to the gap $\Delta$, the one quantity a power spectrum cannot see.
\begin{theorem}[Value ceiling]
\label{thm:vceiling}
Assume $\sigma^2_{\mathrm{lin}}(\mathcal{I})>0$. Let $h$ be measurable with respect to the information set $\mathcal{I}$ and have finite MSE. Then the relative error reduction of $h$ over the best linear predictor on $\mathcal{I}$ satisfies
\begin{equation}
\label{eq:vgap}
\frac{\sigma^2_{\mathrm{lin}}(\mathcal{I})-\mathrm{MSE}(h)}
{\sigma^2_{\mathrm{lin}}(\mathcal{I})}
\le
\frac{\Delta(\mathcal{I})}
{\sigma^2_{\mathrm{lin}}(\mathcal{I})}
=:
\bar V(\mathcal{I}),
\end{equation}
with equality iff $h$ attains the Bayes error $\sigma^2_{*}(\mathcal{I})$. Consequently, for any context-extending mechanism $M$ with access $\mathcal{I}_M$, the relative error reduction of $f\!\oplus\!M$ over the best linear predictor on $\mathcal{I}_M$ is at most $\bar V(\mathcal{I}_M)$.
\end{theorem}
\noindent\emph{Proof sketch.} Every $h$ measurable in $\mathcal{I}$ has $\mathrm{MSE}(h)\ge\sigma^2_{*}(\mathcal{I})$; subtracting from $\sigma^2_{\mathrm{lin}}(\mathcal{I})$ and normalizing gives \eqref{eq:vgap}, with equality exactly at the Bayes error. Randomized predictors are covered by Jensen. Full proof: App.~A.4. Because the bound is over the linear predictor with the same access, $\bar V(\mathcal{I}_M)$ measures the \emph{beyond-second-order} component of context value. For retrieval with memory conditioned upon and the window as key, $\sigma^2_{\mathrm{lin}}(\mathcal{I}_M)=\sigma^2_{\mathrm{lin}}(\mathcal{I}_{\mathrm{base}})$, where $\mathcal{I}_{\mathrm{base}}$ is the base $S$-window. Thus, the mechanism's entire value is beyond second order (App.~A.8). A mechanism that also enlarges the linear information set (a longer window, or a foundation model reading a long context) keeps a spectrum-visible second-order component. The theorem then governs its \emph{margin} over the best linear predictor on that access, which E1 records.

To isolate the beyond-second-order gap, we next introduce surrogate time series. The phase-randomized surrogate $\tilde{x}$ preserves the Fourier amplitudes while randomizing the phases. The iterative amplitude-adjusted Fourier transform (IAAFT) surrogate additionally preserves the marginal distribution (App.~A.1).

\begin{proposition}[Spectral invariance]
\label{prop:invariance}
The periodogram, the full autocovariance, and therefore $\sigma^2_{\mathrm{lin}}(\mathcal{I})$ for every $\mathcal{I}$ drawn from the series are identical for $x$ and its phase-randomized surrogate $\tilde{x}$. Hence any index $P$ that is a functional of the power spectrum or the autocovariance satisfies $P(\tilde{x})=P(x)$; this covers spectral predictability $\Om$ exactly, and spectral-coherence predictability insofar as it reads the preserved per-channel spectra (empirically frozen to $|\Delta\mathrm{SCP}|\le0.015$ in E1). The amplitude-adjusted variant additionally fixes the marginal, up to the reported residual.
\end{proposition}
\looseness=-1 The amplitudes $|X_k|$ are untouched, so the periodogram and its inverse transform, the autocovariance, are preserved at every lag, and $\sigma^2_{\mathrm{lin}}$, which solves the linear normal equations in the autocovariance, follows; App.~A.3 gives the computation.

\begin{lemma}[The phase-randomized surrogate erases the gap]
\label{lem:gap}
Assume the normalized spectral mass is not concentrated on finitely many frequencies (the Lindeberg condition $\max_k a_k^2/s_T^2\to0$). Then the finite-dimensional laws of $\tilde{x}$ converge to those of the stationary Gaussian process $x_G$ with the autocovariance of $x$, second moments are preserved exactly along the sequence, and for every fixed degree $D$ the best degree-$\le D$ polynomial predictor asymptotically gains nothing over the linear one ($\Delta_D(\mathcal{I})\to0$) for every finite $\mathcal{I}$. Under Condition~(M) of App.~A.5 (convergence of conditional means in $L^2$), the full gap closes as well: $\Delta(\mathcal{I})\to0$.
\end{lemma}
\looseness=-1 The surrogate is a sum of independent-phase sinusoids with the covariance of $x$ at every length; the Lindeberg condition kills every standardized joint cumulant of order three and up, so all joint moments converge to Gaussian ones and each fixed-degree least-squares problem converges to its Gaussian counterpart, where the linear predictor is already optimal (App.~A.5). Closing the gap over \emph{all} measurable predictors needs more than moments; Condition~(M) (App.~A.5) supplies it, and no downstream claim uses it: the impossibility is anchored at the exact endpoint $x_G$.

\begin{proposition}[Beyond-spectrum context value is not spectral]
\label{prop:noninvariance}
For the stationary Gaussian process $x_G$ with the autocovariance of $x$, $\bar V(\mathcal{I})=0$ exactly for every $\mathcal{I}$, so by Theorem~\ref{thm:vceiling} the beyond-second-order value of every mechanism is zero on $x_G$: retrieval keyed on the operating window has no value, while a mechanism that also enlarges the linear information set keeps its spectrum-visible second-order gain and loses exactly its margin. For any $x$ with structure beyond second order, $\bar V(\mathcal{I}_M)=\Delta(\mathcal{I}_M)/\sigma^2_{\mathrm{lin}}(\mathcal{I}_M)>0$, and such $x$ exist. Beyond-spectrum context value therefore separates the pair $(x,x_G)$.
\end{proposition}
\noindent\emph{Proof sketch.} Under $x_G$ the coordinates of any window and target are jointly Gaussian, so conditional expectations are affine and $\sigma^2_{*}=\sigma^2_{\mathrm{lin}}$ on every $\mathcal{I}$; window-measurable augmentations do not enlarge the conditioning $\sigma$-algebra (App.~A.8). Existence: for $x_{t+1}=f(x_t)+\varepsilon_t$ with i.i.d.\ noise and $f$ non-affine on the support of the stationary law, $\sigma^2_{*}$ is the noise variance while $\sigma^2_{\mathrm{lin}}$ strictly exceeds it; E3's generator instantiates this. Full proof: App.~A.6.

\begin{corollary}[Impossibility]
\label{cor:impossible}
No predictability index that is a functional of the power spectrum or the autocovariance can determine beyond-spectrum context value: any such $P$ is constant across the pair $(x,x_G)$ (Proposition~\ref{prop:invariance}) while $\bar V$ differs across it whenever $\Delta>0$ (Proposition~\ref{prop:noninvariance}). This covers $\Om$ exactly and SCP up to the reported per-channel residual; the surrogate realizes the comparison at finite length (Lemma~\ref{lem:gap}; proof: App.~A.7).
\end{corollary}

Corollary~\ref{cor:impossible} is the central result. It does not say the indices of Section~\ref{sec:related} are wrong about predictability; it says the deployment question requires a statistic sensitive to the gap $\Delta$ and the window, not the spectrum alone. This refines rather than contradicts prior work: an index reported to predict when foundation models beat baselines \citep{omega2026} is, by Proposition~\ref{prop:invariance}, blind to the gap those models exploit; Section~\ref{sec:method} estimates it directly.

\paragraph{Indices that also read the marginal.}\looseness=-1{} The phase-randomized surrogate alters the marginal, so an index with a distributional term, such as accuracy-law complexity, is not constant across that pair and not covered exactly by Corollary~\ref{cor:impossible}. The IAAFT surrogate fixes the marginal too, holding everything such an index reads; being a static monotonic transform of a Gaussian process rather than Gaussian, its gap is small but nonzero. We compare $\Delta_{\mathrm{nl}}(x)$ against the IAAFT ensemble as a standard surrogate test (E1) and cross-check against phase-randomized (FT) surrogates (E6). A marginal term yields no reliable handle on the gap either.

\begin{remark}
The result applies where context value arises from the gap $\Delta$, the recurring nonlinear motifs and deterministic dynamics that similarity retrieval and in-context completion exploit. It is vacuous where value is purely second-order or $\Delta=0$ leaves nothing to separate. The diagnostic below carries one term for $\Delta$ and one for novelty.
\end{remark}

\section{The Coverage-Deficit Diagnostic}
\label{sec:method}

A useful diagnostic must be computable before deployment, without test labels, and must read the configuration, not only the series. We define the \emph{coverage deficit} $\Gamma(S,H)$ from a coverage term $\Gcov$ and a novelty term $\Goov$, each matched to a way context value goes to zero (Figure~\ref{fig:framework}).

\paragraph{Beyond-spectrum structure term.} The key quantity is the gap $\Delta$ of Section~\ref{sec:theory}: the structure a similarity-retrieval or foundation-model context can exploit and the spectrum cannot represent. We estimate the normalized gap $\Delta/\sigma^2_{\mathrm{lin}}=\bar V$ of Eq.~\eqref{eq:vgap}, label-free on the training split, as the \emph{analog prediction gain}
\begin{equation}
\Delta_{\mathrm{nl}}=1-\frac{\mathrm{MSE}_{\text{analog}}}{\mathrm{MSE}_{\text{linear}}},
\end{equation}
where $\mathrm{MSE}_{\text{linear}}$ estimates $\sigma^2_{\mathrm{lin}}$ with a least-squares predictor and $\mathrm{MSE}_{\text{analog}}$ upper-bounds $\sigma^2_{*}$ with a fixed-$k$ nearest-neighbour predictor, so $\Delta_{\mathrm{nl}}$ is a \emph{conservative} (lower-bound) estimate of $\bar V$, cross-validated within the training split with no test labels (estimator details: App.~A.9). By Proposition~\ref{prop:invariance} and Lemma~\ref{lem:gap}, $\Delta_{\mathrm{nl}}$ is the term the power spectrum cannot see. A series and its phase-randomized surrogate share $\Om$, yet $\Delta_{\mathrm{nl}}$ is large for a series with deterministic motifs and $\to 0$ for the surrogate. The surrogate is the Gaussian process with that spectrum, on which analog matching offers no improvement. The analog predictor is the simplex/nearest-neighbour method of empirical dynamic modeling \citep{sugihara1990,takens1981}, so $\Delta_{\mathrm{nl}}$ is the classical nonlinear-versus-linear prediction gain. This term is what lets $\Gamma$ separate cases the indices treat alike.

\paragraph{Coverage term.} Beyond-spectrum structure is worth supplying only when the window is too short to capture it directly. Let $m$ be the motif length that identifies the local state, with the dominant period $L$ from the training periodogram as the default proxy, and let
\begin{equation}
u(S)=\max\!\Big(0,\ \tfrac{m-S}{m}\Big)
\end{equation}
be the fraction of that motif a length-$S$ window does not observe; a window shorter than the motif cannot form the delay embedding the analog predictor needs \citep{takens1981}, and an input strictly longer than the process memory is necessary even in principle \citep{butera2026}. Then
\begin{equation}
\Gcov(S)=\Delta_{\mathrm{nl}}\cdot u(S).
\end{equation}
$\Gcov$ is large only when there is beyond-spectrum structure to exploit \emph{and} the window is too short to reach it on its own. The operating point enters through $u(S)$; the part the indices miss enters through $\Delta_{\mathrm{nl}}$.

\paragraph{Novelty term.} Even with exploitable structure, a memory is useless if deployment inputs are unlike the training record. Following the symbolic route, discretize each channel into $b$ quantile bins, index training tuples, and let
\begin{equation}
\Goov(S)=\Pr\big[\text{window tuple}\notin \text{index}\big]
\end{equation}
be the out-of-vocabulary rate over the symbolic index, with quantile-bin discretization in the SAX tradition \citep{lin2003sax}; it counts the same event a symbolic training-set memory must handle when a test tuple is absent from its graph \citep{cakiroglu2025dragon}. $\Goov$ is high for memory-hostile, non-recurring distributions, where context value is near zero regardless of structure.

\paragraph{Decision rule.}\looseness=-1{} The predicted sign of context value is a threshold (or logistic) rule on $(\Gcov,\Goov)$, fit on a set of datasets and evaluated leave-one-dataset-out (LODO): high $\Gcov$ and low $\Goov$ predict that context helps. Augmented Dickey--Fuller (ADF) \citep{dickey1979adf} on the training split supplies a trend-domination check where no motif length is well defined, in which case $u(S)\!\to\!0$ and $\Gcov\!\to\!0$ by convention.

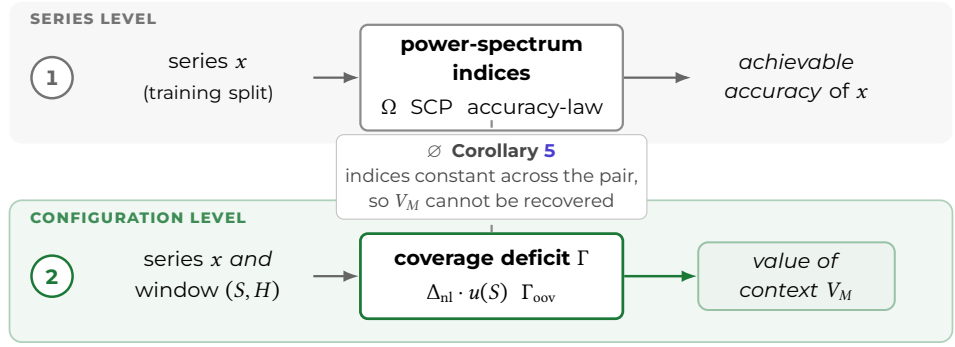
\begin{figure}[t!]
\centering
\resizebox{0.82\columnwidth}{!}{%
\begin{tikzpicture}[
  font=\sffamily\footnotesize,
  >={Latex[length=2.2mm]},
  io/.style={align=center,font=\sffamily\footnotesize,text width=2.7cm},
  proc/.style={rounded corners=3pt,align=center,line width=0.9pt,
               text width=3.35cm,minimum height=1.2cm,inner sep=5pt},
  badge/.style={circle,line width=1.0pt,minimum size=0.58cm,inner sep=0pt,
                font=\sffamily\bfseries\footnotesize},
  lane/.style={rounded corners=5pt,line width=0.6pt},
  llab/.style={font=\sffamily\bfseries\scriptsize,anchor=west},
  arr/.style={->,line width=0.9pt,black!60}]

\begin{scope}[on background layer]
  \fill[black!3,rounded corners=5pt] (-2.35,1.62) rectangle (10.95,3.62);
  \fill[chl!6,rounded corners=5pt] (-2.35,-1.18) rectangle (10.95,0.82);
  \draw[chl!35,line width=0.7pt,rounded corners=5pt] (-2.35,-1.18) rectangle (10.95,0.82);
\end{scope}
\node[llab,text=black!55] at (-2.20,3.38) {\scshape series level};
\node[llab,text=chl!85] at (-2.20,0.58) {\scshape configuration level};

\node[badge,draw=black!55,text=black!60,fill=white] (b1) at (-1.75,2.55) {1};
\node[io] (sx) at (0.45,2.55) {series $x$\\[1pt]{\scriptsize (training split)}};
\node[proc,fill=white,draw=black!45] (P) at (4.45,2.55)
  {\textbf{power-spectrum indices}\\[2pt]$\Om$\ \ SCP\ \ accuracy-law};
\node[io] (acc) at (8.75,2.55) {\emph{achievable}\\ \emph{accuracy} of $x$};
\draw[arr] (sx)--(P);
\draw[arr] (P)--(acc);

\node[badge,draw=chl,text=chl,fill=white] (b2) at (-1.75,-0.25) {2};
\node[io] (sw) at (0.45,-0.25) {series $x$ \emph{and}\\[1pt] window $(S,H)$};
\node[proc,fill=white,draw=chl,line width=1.1pt] (G) at (4.45,-0.25)
  {\textbf{coverage deficit} $\Gamma$\\[2pt]$\Delta_{\mathrm{nl}}\!\cdot\!u(S)$\ \ $\Goov$};
\node[align=center,font=\sffamily\footnotesize,rounded corners=4pt,draw=chl!45,fill=chl!8,
      line width=0.8pt,inner xsep=6pt,inner ysep=4pt,text width=2.35cm] (val) at (8.75,-0.25)
  {\emph{value of}\\ \emph{context} $\Vm_M$};
\draw[arr] (sw)--(G);
\draw[arr,draw=chl,line width=1.0pt] (G)--(val);

\draw[dashed,line width=0.9pt,black!45] (4.45,1.95)--(4.45,0.35);
\node[draw=black!25,rounded corners=3pt,fill=white,align=center,inner sep=3.5pt,
      font=\sffamily\scriptsize,text=black!72] at (4.45,1.15)
  {$\varnothing$\ \ \textbf{Corollary~\ref{cor:impossible}}\\ indices constant across the pair,\\ so $\Vm_M$ cannot be recovered};
\end{tikzpicture}}
\caption{\textbf{Two levels, two questions.} Power-spectrum indices read the series and predict its achievable accuracy (top band). The coverage deficit reads the series \emph{and} the window and predicts the value of context (bottom band, highlighted). The dashed link is the impossibility result: the top-band indices are constant across a phase-randomization pair whose context value differs (Corollary~\ref{cor:impossible}), so the bottom-band quantity cannot be recovered from them.}
\label{fig:framework}
\end{figure}

\section{Experimental Protocol}
\label{sec:experiments}

The protocol tests, in order, that the spectrum-controlled gap is real (Proposition~\ref{prop:noninvariance}), that the diagnostic predicts context-value sign where the indices cannot (Corollary~\ref{cor:impossible}), and that both hold across mechanisms, each the simplest standard instance of its class under one protocol. The protocol uses seven benchmarks (D7: ETTh1/h2, ETTm1/m2, Weather, ECL, Traffic), operating windows $S\in\{12,24,48,96\}$, and direct multi-step prediction at the four standard horizons $H\in\{96,192,336,720\}$ (detailed tables at $H{=}96$; the collapse is verified at all four). It uses three seeds for the surrogate draws, paired MSE on $z$-normalized channels with $\Vm_M=(\mathrm{MSE}(f)-\mathrm{MSE}(f{\oplus}M))/\mathrm{MSE}(f)$ per cell, and the last $20{,}000$ points per channel. The base forecaster $f$ is the direct-$H$ least-squares predictor on the $S$-window. The mechanisms are (a) a longer linear window ($4S$ lags; purely second-order), (b) analog retrieval keyed on the $S$-window over the training record (the simplex predictor of \citealp{sugihara1990}; adds no linear information), and (c) a zero-shot foundation model (Chronos-Bolt, \citealp{ansari2024chronos}) reading a long context of $512$ points. For mechanism (c), Theorem~\ref{thm:vceiling} bounds the population margin over the best linear predictor on the same access; we report its empirical counterpart, the margin over the train-fit linear predictor on the same context window.

\paragraph{E1. The spectrum-controlled gap.}
For each benchmark channel and each $S$, generate $K{=}20$ amplitude-adjusted surrogates (IAAFT, $1000$ iterations) per series, preserving the periodogram and marginal; a cell whose mean periodogram residual exceeds $0.02$ is excluded as not spectrum-controlled and counted. Measure $\Vm_M$ on the original and on every surrogate for the three mechanisms, the foundation model on an eight-channel-per-dataset subsample with $K{=}10$. Report $\Om$, SCP (identical across arms by Proposition~\ref{prop:invariance}, up to the residual), and $\Gcov$ for both arms. Statistics are medians over channels, a bootstrap 95\% CI on the median paired gap, and a one-sided Wilcoxon signed-rank over paired cells. \emph{E1 is the positive control for the whole argument: absent a spectrum-controlled gap, none of the downstream claims can hold.}

\paragraph{E2. Sign prediction, head to head.}
Leave-one-dataset-out prediction of the sign of $\Vm_M$ across all (dataset, channel, $S$) cells. Each rule is a one-dimensional threshold on its statistic, with the threshold and direction fit on the six training datasets by balanced accuracy: the structure term $\Delta_{\mathrm{nl}}$ at the operating window (a motif-embedding variant is compared qualitatively in Limitations), $\Gcov$, $\Om$ \citep{omega2026}, SCP \citep{scp2025}, bicoherence \citep{nikias1987bispectrum}, permutation entropy \citep{bandt2002permutation}, catch22 with gradient boosting \citep{lubba2019catch22}, and a per-fold majority baseline. Metric: balanced sign accuracy; significance by a label-permutation test against chance ($B{=}2000$). It also tests whether a phase-sensitive higher-order index or generic features would suffice.

\paragraph{E3. Correlation with measured value.}
Spearman correlation of each index with measured $\Vm_M$ across cells, on the original arm and on the matched surrogate arm. Supports Corollary~\ref{cor:impossible}.

\paragraph{E4. Generality across mechanisms.}
Repeat E2 separately for $M=$ longer lookback, $M=$ retrieval, and $M=$ foundation model, the last for its beyond-linear margin (Theorem~\ref{thm:vceiling}). The theory predicts the boundary: one $\Delta_{\mathrm{nl}}$ rule should predict the sign of every \emph{beyond-spectrum} component, while spectral rules should predict the purely second-order mechanism.

\paragraph{E5. Out-of-distribution transfer.}
Freeze the $\Gamma$ rule and the index rules fit on the seven benchmarks; evaluate on withheld Exchange, ILI, and M5 retail series, none seen in development.

\paragraph{E6. Ablations.}
$\Gcov$ with $\Delta_{\mathrm{nl}}$ forced to one; $\Gcov$ and $\Goov$ alone; sensitivity to the surrogate count $K$ and the neighbor count; a cross-check of FT against IAAFT surrogates, which lack the remapping artifact \citep{rath2009surrogates}; and MAE--MSE agreement on in-regime cells.

\section{Results}
\label{sec:results}

\paragraph{The gap is real and spectral indices are blind to it (E1, E3).}\looseness=-1{} Table~\ref{tab:gap} is the paper's spine: on the two benchmarks where window-keyed retrieval has material value, the value collapses to near or below zero. The median $\Vm_M$ falls from $+33.0\%$ to $-35.0\%$ on ECL and from $+33.8\%$ to $+0.7\%$ on Traffic, passing through zero to the analog estimator's negative finite-sample floor as the Gaussian endpoint of Lemma~\ref{lem:gap} predicts, and each paired gap is large and overwhelmingly significant. The effect is broad ($86\%$ of ECL and $93\%$ of Traffic channels carry positive value; the other five benchmarks have honestly negative medians), seed-stable, and horizon-stable (Table~\ref{tab:horizon}). The invariance is not approximate: across every ECL and Traffic surrogate pair the spectral indices are frozen to two decimal places ($|\Delta\Om|\le0.014$, $|\Delta\mathrm{SCP}|\le0.015$), while the context value they are meant to predict swings by up to seventy points, the empirical face of Corollary~\ref{cor:impossible}. Nor is the collapse a remapping artifact: plain phase-randomized (FT) surrogates reproduce it (App.~A.10), and forcing $\Delta_{\mathrm{nl}}{=}1$ so that coverage acts alone drops sign agreement with retrieval value from $0.90$ to $0.09$, confirming that the structure term, not the operating-point fraction, carries the signal. That term also tracks the collapse per dataset, agreeing in sign with the measured change in $\Vm_M$ on all seven benchmarks (four at $\ge0.90$ cell-level agreement, Traffic at $0.90$, the lowest $0.67$). The longer \emph{linear} window, by contrast, survives phase randomization intact (ECL $56.7\%\to55.1\%$), as Proposition~\ref{prop:noninvariance} requires, and the foundation model splits into a dominant second-order part that survives (ECL: $+66\%$ of its $+77\%$ total) and a small beyond-linear \emph{margin} that collapses wherever the model carries one (Table~\ref{tab:fmmargin}). For E3, the structure term correlates with measured retrieval value across all $4{,}622$ gated cells, with Spearman $\rho$ rising from $0.75$ on the original arm to $0.90$ on the surrogate arm, while $\Om$ carries no signal ($\rho{=}0.03$) and SCP is negatively correlated throughout.

\begin{table}[t]
\centering
\footnotesize
\setlength{\tabcolsep}{3.5pt}
\renewcommand{\arraystretch}{0.96}
\setlength{\abovecaptionskip}{3pt}
\setlength{\belowcaptionskip}{-2pt}
\begin{tabular}{l rr r c r}
\toprule
& \multicolumn{2}{c}{$\Vm_M$ (\%)} & \multicolumn{3}{c}{paired gap}\\
\cmidrule(lr){2-3}\cmidrule(lr){4-6}
Dataset & \multicolumn{1}{c}{$x$} & \multicolumn{1}{c}{$\tilde{x}$} & \multicolumn{1}{c}{pts} & 95\% CI & \multicolumn{1}{c}{$p$}\\
\midrule
ECL     & $+33.0$ & $-35.0$ & $\mathbf{69.8}$ & $[64.5,\,73.1]$ & $3.5{\times}10^{-47}$\\
Traffic & $+33.8$ & $+0.7$  & $\mathbf{32.4}$ & $[31.2,\,33.2]$ & $3.4{\times}10^{-125}$\\
\bottomrule
\end{tabular}
\caption{\textbf{The spectrum-controlled gap} on the two benchmarks where retrieval has material value ($S{=}12$, $H{=}96$, $M=$ window-keyed retrieval; channel medians, $K{=}20$ IAAFT; $n{=}301$ ECL and $813$ Traffic channels). $\Vm_M$ collapses through zero under phase randomization, whereas the spectral index $\Om$ is unchanged by construction (Prop.~\ref{prop:invariance}). Gaps are paired medians with bootstrap 95\% confidence intervals.}
\label{tab:gap}
\end{table}

\begin{table}[t]
\centering
\footnotesize
\setlength{\tabcolsep}{6pt}
\renewcommand{\arraystretch}{0.96}
\setlength{\abovecaptionskip}{3pt}
\setlength{\belowcaptionskip}{-2pt}
\begin{tabular}{l rrrr}
\toprule
& \multicolumn{4}{c}{retrieval gap (pts) at horizon $H$}\\
\cmidrule(lr){2-5}
Dataset & \multicolumn{1}{c}{96} & \multicolumn{1}{c}{192} & \multicolumn{1}{c}{336} & \multicolumn{1}{c}{720}\\
\midrule
ECL         & 69.8 & 66.6 & 64.5 & 62.4\\
Traffic     & 32.4 & 31.3 & 30.4 & 29.7\\
D7 (pooled) & 35.0 & 33.7 & 32.9 & 31.8\\
\bottomrule
\end{tabular}
\caption{\textbf{The collapse holds at every standard horizon} ($S{=}12$; paired median, points). It declines only gently with $H$ and stays overwhelmingly significant ($p<10^{-46}$ on ECL, $p<10^{-124}$ on Traffic, at all four).}
\label{tab:horizon}
\end{table}

\begin{table}[t]
\centering
\footnotesize
\setlength{\tabcolsep}{3.5pt}
\renewcommand{\arraystretch}{0.96}
\setlength{\abovecaptionskip}{3pt}
\setlength{\belowcaptionskip}{-2pt}
\begin{tabular}{l rr r c r}
\toprule
& \multicolumn{2}{c}{margin (\%)} & \multicolumn{3}{c}{paired gap}\\
\cmidrule(lr){2-3}\cmidrule(lr){4-6}
Dataset & \multicolumn{1}{c}{$x$} & \multicolumn{1}{c}{$\tilde{x}$} & \multicolumn{1}{c}{pts} & 95\% CI & \multicolumn{1}{c}{$p$}\\
\midrule
ECL     & $+9.0$  & $-8.8$  & $\mathbf{18.5}$ & $[13.2,\,37.2]$ & $0.016$\\
Traffic & $+10.9$ & $-10.8$ & $\mathbf{15.3}$ & $[3.5,\,30.6]$  & $0.012$\\
\bottomrule
\end{tabular}
\caption{\textbf{The foundation model's beyond-linear margin collapses} under phase randomization on both benchmarks where the model carries one ($S{=}12$, $H{=}96$, $M=$ Chronos-Bolt on a $512$-point context; channel medians, $K{=}10$ IAAFT). The margin is the value over the train-fit linear predictor on the same context window, the empirical counterpart of the bound in Thm.~\ref{thm:vceiling}; the total value ($+77\%$ median on ECL) is dominated by a second-order part that survives ($+66\%$). The paired gap is also significant on ETTm2 ($p{=}0.031$).}
\label{tab:fmmargin}
\end{table}

\paragraph{The diagnostic predicts the sign where indices do not (E2, E4).}\looseness=-1{} Table~\ref{tab:lodo} reports leave-one-dataset-out sign accuracy, and the boundary matches the split Corollary~\ref{cor:impossible} predicts. One structure-term rule predicts both beyond-spectrum components ($78.1\%$ retrieval, $73.2\%$ foundation margin; both permutation-significant, caption), where $\Om$ is below chance on retrieval ($39.5\%$) and at chance on the foundation margin ($52.4\%$). The mirror image holds for the purely second-order mechanism, where SCP reaches $79.2\%$ and the structure term is rightly silent. The comparison also answers a natural objection: since our impossibility covers only power-spectrum functionals, would a phase-sensitive higher-order index suffice? Bicoherence, a third-order phase-coupling measure, does beat $\Om$ on retrieval ($66.8$ vs $39.5\%$), confirming it sees structure the power spectrum cannot. But it, permutation entropy, a catch22 gradient-boosting stack, and the two newest indices we cite (accuracy-law complexity and Catt's profile) all remain series-level and trail the configuration-level structure term on both beyond-spectrum mechanisms (caption).

\begin{table}[t]
\centering
\footnotesize
\setlength{\tabcolsep}{5pt}
\renewcommand{\arraystretch}{0.96}
\setlength{\abovecaptionskip}{3pt}
\setlength{\belowcaptionskip}{-2pt}
\begin{tabular}{l rrr}
\toprule
& \multicolumn{3}{c}{balanced sign accuracy (\%)}\\
\cmidrule(lr){2-4}
Predictor & \multicolumn{1}{c}{lookback} & \multicolumn{1}{c}{retrieval} & \multicolumn{1}{c}{foundation}\\
\midrule
\rowcolor{black!5}
\textbf{$\Delta_{\mathrm{nl}}$ (ours)} & 46.9 & \textbf{78.1} & \textbf{73.2}\\
$\Om$            & 24.0 & 39.5 & 52.4\\
SCP              & \textbf{79.2} & 66.1 & 54.5\\
bicoherence      & 46.7 & 66.8 & 67.3\\
\addlinespace[2pt]
majority         & 50.0 & 50.0 & 26.4\\
\bottomrule
\end{tabular}
\caption{\textbf{Leave-one-dataset-out balanced sign accuracy} for predicting the sign of context value across mechanisms ($S{=}12$, $H{=}96$). $\Delta_{\mathrm{nl}}$ performs best on the two beyond-spectrum mechanisms (retrieval and foundation-model margin), while SCP performs best on the second-order lookback. The foundation column reports the beyond-linear margin (Thm.~\ref{thm:vceiling}); \emph{majority} is the per-fold majority-class baseline.}
\label{tab:lodo}
\end{table}

\paragraph{Transfer (E5).} \looseness=-1 Rules frozen on the seven development benchmarks and applied to withheld Exchange, ILI, and M5 (untouched in development) meet a finding of their own. Retrieval carries beyond-spectrum value in almost none of their cells (positive in $0\%$ of Exchange cells, $4\%$ of ILI, $22\%$ of M5), so the correct call is overwhelmingly ``context will not help.'' The in-distribution prior, carrying the development ``helps'' majority, is therefore wrong almost everywhere ($0$--$22\%$ accuracy), the exact failure we warn against. The frozen structure-term rule beats it on raw sign accuracy on all three sets and correctly calls Exchange's near-total absence of value. It is not a clean sweep: on balanced accuracy the structure term clears chance only on M5 and falls below it on ILI's 28 cells (Exchange has no positive cells, so balanced accuracy there is not a chance comparison), and we read E5 as an honest specificity test the diagnostic mostly passes, not a second positive control (full table in App.~A.12).

\section{Discussion}

\looseness=-1 The mechanism split reconciles our impossibility result with the strongest empirical finding in this literature. \citet{omega2026} report that $\Om$ predicts when foundation models beat light baselines; our decomposition explains why both facts hold at once. The foundation model's value on these benchmarks is predominantly second-order, content the spectrum represents and $\Om$ can rank. Its beyond-linear margin, the part Corollary~\ref{cor:impossible} says no spectral index can see, is several times smaller ($5.7\times$ on Traffic, $7.3\times$ on ECL) and collapses under phase randomization wherever the model carries it. The bound constrains what a spectral index can see, not any one model, so it holds for current foundation-model successors \citep{liu2025moirai2, chronos2} as well, which the frontier Chronos-2 confirms: its beyond-linear margin likewise collapses on ECL and Traffic (both $p<0.02$). Retrieval sits at the opposite pole: for mechanisms whose promise is beyond-spectrum structure, a spectral index is provably, and now measurably, silent.

\section{Limitations}

\looseness=-1 The coverage term relies on a motif length (the trend-robust first-difference period), degrading on multi-scale states; $\Delta_{\mathrm{nl}}$ inherits nearest-neighbor sensitivity to embedding and length. The novelty term is inert in-distribution and earns its keep only under shift. IAAFT preserves the periodogram up to a gated residual (gate $0.02$), cross-checked against FT surrogates \citep{rath2009surrogates}; the Gaussianization is asymptotic and degree-wise (App.~A.5), excluding finitely supported spectra. Our probes use a linear direct-$H$ base and, per class, the simplest clean mechanism (the analog retriever adds no linear predictability, so Theorem~\ref{thm:vceiling} attributes its whole value to the gap). A learned SOTA retriever \citep{han2025raft,specretf2026} may add linear predictability and is left as a consistent extension; the foundation-model margins rest on eight-channel Chronos-Bolt and frontier Chronos-2 subsamples. The collapse survives a nonlinear MLP base ($p<0.01$ on ECL and Traffic), not just the linear one; deep SOTA backbones \citep{zeng2023dlinear,nie2023patchtst,liu2024itransformer} and frontier models \citep{liu2025moirai2} remain untested. Channels are treated univariately; the multivariate surrogate \citep{prichard1994} covers the cross-channel case. The theory assumes second-order stationarity, and channels are treated as exchangeable: nonstationarity is absorbed only partially (ADF gate, $z$-normalization), and cross-channel dependence makes the CIs and $p$-values nominal rather than conservative.

\section{Conclusion}

\looseness=-1 Series-level predictability and the value of added context are different quantities: one belongs to the series, the other to the operating point. Because every power-spectrum index is invariant under phase randomization while beyond-second-order context value is not, none can decide the deployment question. We isolate the gap with surrogate pairs that fix spectrum and marginal by construction, and fill it with the coverage-deficit diagnostic: computed pre-deployment without labels, its structure term predicts the sign of beyond-spectrum value leave-one-dataset-out where spectral indices trail it. The boundary is honest: second-order value is already spectrum-visible, and the diagnostic reports whether context helps, not how much.

\begin{availability}
The surrogate generator, the coverage-deficit diagnostic, every predictability index, and
the experiment runners are released at
\ifbool{KIL@blind}
  {\url{https://anonymous.4open.science/r/SINE}}
  {\url{https://github.com/KurbanIntelligenceLab/SINE}},
with a smoke test and seeded surrogate draws. All datasets (ETT, Weather, ECL, Traffic,
Exchange, ILI, and M5) are public and cited in \cref{sec:experiments}; no new human data
were collected.
\end{availability}

\begin{conflicts}
The authors declare no competing interests.
\end{conflicts}

\bibliography{phasepower}

\clearpage
\appendix
\section{Surrogates, Proofs, and Estimator}
\label{app:proofs}

This appendix supplies the surrogate construction (A.1); the autocovariance representation of the linear floor (A.2); complete proofs of spectral invariance (A.3), the value ceiling (A.4), the gap-erasure lemma with the explicit Condition~(M) for full Bayes-risk convergence (A.5), beyond-spectrum non-invariance (A.6), and the impossibility corollary (A.7); the formal retrieval information set (A.8); the gap estimator (A.9); and the ablation and full-table material (A.10--A.12). Every main-text result states its assumptions in place and is sketched there; every proof below is complete.

\subsection{Phase randomization and IAAFT}
\begin{definition}[Phase randomization]
For a finite real series $x$ of length $T$ with discrete Fourier transform $X_k=|X_k|e^{i\phi_k}$, a \emph{phase-randomized surrogate} $\tilde{x}$ has transform $|X_k|e^{i\psi_k}$ with the $\psi_k$ drawn i.i.d.\ uniform on $[0,2\pi)$, subject to conjugate symmetry so that $\tilde{x}$ is real. The \emph{amplitude-adjusted} (IAAFT) variant additionally rank-maps the surrogate onto the empirical values of $x$.
\end{definition}
The IAAFT iteration \citep{schreiber2000}: given sorted values $c$ of $x$ and target amplitudes $A=|\mathrm{DFT}(x)|$, initialize $s$ as a random permutation of $x$, then alternate (1) $s\leftarrow\mathrm{IDFT}\!\big(A\,e^{i\angle\mathrm{DFT}(s)}\big)$ to impose the spectrum and (2) a rank-map of $s$ onto $c$ to impose the marginal, until the relative amplitude change falls below tolerance. We report the residual $\varepsilon=\big\|\,|\mathrm{DFT}(s)|-A\,\big\|/\|A\|$. As a cross-check we also use plain phase-randomized (FT) surrogates, which omit the rank-map and therefore carry no remapping-induced phase correlations \citep{rath2009surrogates}.

\subsection{The linear floor is an autocovariance functional}
\label{app:linfloor}
Let $\gamma(h)=\mathrm{cov}(x_t,x_{t+h})$. For one-step prediction ($H{=}1$) from an information set $\mathcal{I}$ that is a finite collection of coordinates of $x$, the best linear predictor solves the normal equations $\Sigma_{\mathcal{I}}\,b=c$, where $\Sigma_{\mathcal{I}}$ stacks the $\gamma(\cdot)$ among the coordinates of $\mathcal{I}$ and $c$ stacks the $\gamma(\cdot)$ between $\mathcal{I}$ and the target. The resulting error $\sigma^2_{\mathrm{lin}}(\mathcal{I})=\gamma(0)-c^{\top}\Sigma_{\mathcal{I}}^{-1}c$ depends on $x$ \emph{only} through $\{\gamma(h)\}$, hence only through the power spectrum; the $H$-step case replaces this scalar expression with the trace of the analogous block form and is identical in its dependence on $\{\gamma(h)\}$.

\subsection{Proof of spectral invariance}
Since $\big||X_k|e^{i\psi_k}\big|^2=|X_k|^2$, the periodogram $I(k)=|X_k|^2/T$ is unchanged for every $k$, and the \emph{circular} sample autocovariance, its inverse transform, is preserved at every lag (the ordinary sample autocovariance differs only in edge terms of order $1/T$), including the past--future lags that link a window to its horizon. By Appendix~\ref{app:linfloor}, $\sigma^2_{\mathrm{lin}}(\mathcal{I})$ is then identical for $x$ and $\tilde{x}$ for every $\mathcal{I}$. Any $P$ that is a measurable functional of $\{I(k)\}$ or $\{\gamma(h)\}$ inherits the invariance: spectral concentration and entropy ($\Om$) exactly; per-channel spectral-coherence quantities exactly; cross-channel magnitudes $|X_k\bar Y_k|=|X_k||Y_k|$ per frequency exactly, while smoothed cross-coherence, which averages random relative phases, is preserved only up to the reported residual; and the covariance component of window-wise complexity. The IAAFT rank-map fixes the empirical marginal exactly, at the cost of the periodogram residual $\varepsilon$; hence any functional of the power spectrum and the marginal is invariant up to $\varepsilon$. \hfill$\square$

\subsection{Proof of the value ceiling}
Every predictor $h$ measurable with respect to $\mathcal{I}$ has $\mathrm{MSE}(h)\ge\sigma^2_{*}(\mathcal{I})$, the infimum over all measurable predictors. Subtracting from $\sigma^2_{\mathrm{lin}}(\mathcal{I})$ and using $\sigma^2_{*}(\mathcal{I})=\sigma^2_{\mathrm{lin}}(\mathcal{I})-\Delta(\mathcal{I})$ gives $\sigma^2_{\mathrm{lin}}(\mathcal{I})-\mathrm{MSE}(h)\le\Delta(\mathcal{I})$; dividing by $\sigma^2_{\mathrm{lin}}(\mathcal{I})>0$ yields the bound, with equality iff $\mathrm{MSE}(h)=\sigma^2_{*}(\mathcal{I})$. A randomized predictor $h(\mathcal{I},U)$ with auxiliary noise $U$ independent of $(\mathcal{I},\text{target})$ satisfies $\mathrm{MSE}(h)\ge\mathrm{MSE}(\mathbb{E}[h\mid\mathcal{I}])\ge\sigma^2_{*}(\mathcal{I})$ by Jensen, so the bound covers stochastic forecasters. Taking $\mathcal{I}=\mathcal{I}_M$ and $h=f{\oplus}M$ gives the mechanism statement. \hfill$\square$

\subsection{Proof of the gap-erasure lemma}
Write the surrogate as $\tilde{x}_t=\sum_{k=1}^{m} a_k\cos(2\pi k t/T+\psi_k)$ with $a_k=|X_k|$, the $\psi_k$ i.i.d.\ uniform on $[0,2\pi)$, and the mean removed ($a_0{=}0$).

\emph{Second moments are exact.} For any lag $h$, taking the expectation over the independent uniform phases kills every cross-frequency term and leaves $\mathbb{E}[\tilde{x}_t\tilde{x}_{t+h}]=\tfrac12\sum_k a_k^2\cos(2\pi k h/T)=\gamma(h)$, the target autocovariance, at every $T$. Hence the covariance of any finite coordinate vector of $\tilde{x}$ equals that of $x$ exactly, and by Appendix~\ref{app:linfloor} $\sigma^2_{\mathrm{lin}}(\mathcal{I})$ is the same for $x$ and $\tilde{x}$.

\emph{Higher cumulants vanish.} Fix a finite index set $\{t_1,\dots,t_n\}$ and set $s_T^2=\tfrac12\sum_k a_k^2$. Because the contributions from distinct frequencies are independent and cumulants of a sum of independent terms add, the joint cumulant of order $r$ of the coordinates is $\sum_k$ of the order-$r$ cumulant of the frequency-$k$ term, each bounded in modulus by $C_r\,a_k^{\,r}$ with $C_r$ absorbing the bounded trigonometric factor. Standardizing by $s_T^r$ and using $\sum_k a_k^{\,r}\le(\max_k a_k)^{r-2}\sum_k a_k^2=2s_T^2(\max_k a_k)^{r-2}$,
\[
\frac{|\mathrm{cum}_r|}{s_T^{\,r}}\ \le\ 2\,C_r\Big(\max_k\frac{a_k^2}{s_T^2}\Big)^{\!(r-2)/2}\ \xrightarrow{\;T\to\infty\;}\ 0
\]
for every $r\ge3$ under the no-dominant-frequency (Lindeberg) condition $\max_k a_k^2/s_T^2\to0$, while the order-two cumulants stay fixed at $\gamma(\cdot)$. By the multivariate Lindeberg--Feller theorem the finite-dimensional laws of $\tilde{x}$ converge to those of the Gaussian process with covariance $\gamma(\cdot)$.

\emph{From law to risk.} For the Gaussian endpoint $x_G$ a conditional expectation is affine, so $\sigma^2_{*}(\mathcal{I})=\sigma^2_{\mathrm{lin}}(\mathcal{I})$ and the gap is \emph{exactly} zero; Proposition~\ref{prop:noninvariance} and Corollary~\ref{cor:impossible} use only this endpoint, so the impossibility needs no asymptotics. Along the sequence, fix a degree $D$ and the polynomial feature map of degree at most $D$ on the standardized (window, target) coordinates. The degree-$D$ least-squares risk is a fixed polynomial in the joint moments of order at most $2D{+}2$; cumulant convergence gives moment convergence, and the Gaussian Gram matrix of the feature map is nonsingular for a nondegenerate covariance, so the optimal degree-$D$ risk converges to its Gaussian value, which the affine predictor already attains. Hence $\Delta_D(\mathcal{I})=\sigma^2_{\mathrm{lin}}(\mathcal{I})-\inf_{\deg\le D}\mathrm{MSE}\to0$ for every fixed $D$.

\emph{Condition (M).} Termwise decay does not by itself close the gap over \emph{all} measurable predictors: $\Delta(\mathcal{I})\to0$ additionally requires the conditional means to converge, $\mathbb{E}[Y_T\mid X_T]\to\mathbb{E}_G[Y\mid X]$ in $L^2$ (Condition~(M)); a local limit theorem for the standardized vector together with uniform integrability of $Y_T^2$ suffices. We state (M) explicitly rather than assume it silently; no main-text claim uses it, and E1 measures the finite-length gap directly. A finitely supported spectrum ($m$ lines) is excluded by the Lindeberg condition, and rightly so: such a series obeys an order-$2m$ linear recurrence, so for $S\ge2m$ it is perfectly linearly predictable and carries no gap, while for $S<2m$ it carries a genuine gap that phase randomization does not erase. \hfill$\square$

\subsection{Proof of beyond-spectrum non-invariance}
Under $x_G$, any finite collection of series coordinates is jointly Gaussian with the target, so $\mathbb{E}[Y\mid\mathcal{I}]$ is affine and $\sigma^2_{*}(\mathcal{I})=\sigma^2_{\mathrm{lin}}(\mathcal{I})$, i.e.\ $\bar V(\mathcal{I})=0$, for every such $\mathcal{I}$. A window-keyed retrieval reads a $\sigma(\kappa)$-measurable augmentation (Appendix~A.8), so it does not enlarge the conditioning $\sigma$-algebra and, by Theorem~\ref{thm:vceiling} at $\bar V=0$, has value $0$ on $x_G$; a mechanism that widens the linear span retains its second-order gain and, again by Theorem~\ref{thm:vceiling}, loses exactly its beyond-linear margin. Existence: let $x_{t+1}=f(x_t)+\varepsilon_t$ with $\varepsilon_t$ i.i.d., mean zero, variance $\sigma_\varepsilon^2$, independent of the past, and $f$ bounded with a stationary solution. Then $\mathbb{E}[x_{t+1}\mid x_t]=f(x_t)$, so $\sigma^2_{*}=\sigma_\varepsilon^2$, while $\sigma^2_{\mathrm{lin}}=\sigma_\varepsilon^2+\min_{a,b}\mathbb{E}\,(f(x_t)-a-bx_t)^2>\sigma_\varepsilon^2$ whenever $f$ is not almost surely affine on the support of the stationary law; hence $\Delta>0$ and $\bar V>0$ on the one-step window. E3's generator instantiates this construction. \hfill$\square$

\subsection{Proof of the impossibility corollary}
Let $P$ be any functional of the power spectrum or the autocovariance. By Proposition~\ref{prop:invariance}, $P$ depends on the process only through $\{\gamma(h)\}$, so $P(x)=P(x_G)$ for the Gaussian process $x_G$ sharing that autocovariance. By Proposition~\ref{prop:noninvariance}, $\bar V=0$ on $x_G$ while $\bar V(\mathcal{I}_M)>0$ for any $x$ with $\Delta(\mathcal{I}_M)>0$, and such $x$ exist. A single value $P(x)=P(x_G)$ cannot determine a quantity that differs across the pair, so no such $P$ determines beyond-spectrum context value. The statement covers $\Om$ exactly and SCP up to the reported per-channel residual. \hfill$\square$

\subsection{The retrieval information set}
\label{app:retrieval}
A retrieval mechanism carries a fixed training memory $\mathcal{M}$ (built once, not re-estimated at test time) and reads a key $\kappa$ at prediction time; its information set is $\mathcal{I}_M=\sigma(\kappa)$ with $\mathcal{M}$ conditioned upon, and $\sigma^2_{\mathrm{lin}}(\mathcal{I}_M)$ is the residual of the best affine function of $\kappa$. When the key is the operating window, $\kappa=x_{t-S+1:t}$ and $\sigma^2_{\mathrm{lin}}(\mathcal{I}_M)=\sigma^2_{\mathrm{lin}}(\mathcal{I}_{\mathrm{base}})$: the memory contributes no linear predictability beyond the window's own, so by Theorem~\ref{thm:vceiling} the mechanism's entire relative value is the beyond-second-order margin $\bar V(\mathcal{I}_{\mathrm{base}})$, which Lemma~\ref{lem:gap} sends to zero on the surrogate. A mechanism that also widens the linear span (a longer window, or a long-context model reading more than $S$ points) instead keeps a spectrum-visible second-order component that phase randomization preserves; Theorem~\ref{thm:vceiling} then bounds only its \emph{margin} over the best linear predictor on that wider access, and E1 reports that margin.

\subsection{Analog-gain estimator \texorpdfstring{$\Delta_{\mathrm{nl}}$}{Delta-nl}}
Delay-embed $x$ at dimension $d$ (default $d=\mathrm{clip}(L,4,32)$ for dominant period $L$) and split the embedded pairs $60/40$ in time. $\mathrm{MSE}_{\text{linear}}$ is the one-step error of the least-squares AR$(d)$ predictor fit on the first split and tested on the second, an estimator of $\sigma^2_{\mathrm{lin}}$. $\mathrm{MSE}_{\text{analog}}$ is the error of a fixed-$k$ nearest-neighbour predictor ($k=4$) over the same library; being one specific measurable predictor it satisfies $\mathrm{MSE}_{\text{analog}}\ge\sigma^2_{*}$. Hence, at the population level where $\mathrm{MSE}_{\text{linear}}\!\to\!\sigma^2_{\mathrm{lin}}$, $\Delta_{\mathrm{nl}}=\mathrm{clip}\big(1-\mathrm{MSE}_{\text{analog}}/\mathrm{MSE}_{\text{linear}},0,1\big)\le 1-\sigma^2_{*}/\sigma^2_{\mathrm{lin}}=\bar V$ is a conservative (lower-bound) estimate of the normalized gap of Theorem~\ref{thm:vceiling}; we do not claim consistency at fixed $k$. It is the standard nonlinear-prediction statistic of surrogate-data analysis \citep{theiler1992,schreiber2000}; under the linear-Gaussian null of Lemma~\ref{lem:gap} (the phase-randomized surrogate) it is $\approx 0$ in expectation, and against the IAAFT ensemble it is calibrated as a surrogate test (Appendix above), which is why it separates a structured series from its surrogates.

The primary E2/E4 rule (Table~\ref{tab:lodo}) is the operating-window structure term: $\Delta_{\mathrm{nl}}$ with delay-embedding dimension $d{=}S$, so $d{=}12$ at the reported operating point (column \texttt{dnl\_S\_train} in the released shards). The motif-embedding variant instead sets $d$ from the first-difference motif length $m$ ($d{=}\mathrm{clip}(m,4,96)$, column \texttt{delta\_nl}) and is the form compared qualitatively in the main-text Limitations. The released \texttt{analysis\_phasepower.py} exposes both through the embedding argument (pass the window length $S$ for the operating-window primary).

\subsection{Ablation numbers (E6)}
All at $S{=}12$, $H{=}96$, retrieval mechanism unless noted.
\emph{Structure vs.\ coverage.} Sign agreement of a single rule with $\mathrm{sign}(V{>}0)$ over all seven-benchmark (D7) cells: the operating-window $\Delta_{\mathrm{nl}}$ alone $0.90$, $\Gcov$ alone $0.73$, and the coverage fraction $u(S)$ alone (i.e.\ $\Delta_{\mathrm{nl}}$ forced to $1$) only $0.09$. The structure term, not coverage, carries the signal.
\emph{FT vs.\ IAAFT surrogates.} Median retrieval $V$ (original\,$\to$\,surrogate): FT gives ECL $+0.43\to-0.32$ and Traffic $+0.37\to-0.03$; IAAFT gives $+0.43\to-0.26$ and $+0.37\to+0.02$ on the same 52-channel subsample. The collapse is present under both, so it is not an IAAFT remapping artifact \citep{rath2009surrogates}.
\emph{Surrogate count.} On the gated Table~\ref{tab:gap} population, the median paired retrieval gap on ECL is $+0.667,+0.686,+0.698$ for $K=5,10,20$ ($n{=}301$) and on Traffic $+0.322,+0.323,+0.324$ ($n{=}813$); the estimate is stable in $K$ and reaches the main-text Table~\ref{tab:gap} value at $K{=}20$.
\emph{Neighbours and metric.} Median $V$ rises monotonically with the neighbour count ($k{=}2,4,8$) on every dataset without changing sign, and the sign of $V$ under MAE agrees with that under MSE in every cell on ECL, ETTh2, ETTm2 and Traffic, $0.86$ on ETTh1, and $0.71$/$0.75$ on ETTm1/Weather.

\subsection{Full spectrum-controlled gap table (E1)}
Table~\ref{tab:gapfull} gives the per-dataset E1 numbers for all seven benchmarks
(the main text shows the two with material retrieval value). $\Vm_M$ is the
window-keyed retrieval value at $S{=}12$, $H{=}96$; medians over surrogate-valid
channels, $K{=}20$ IAAFT. $\Om(x){=}\Om(\tilde x)$ by construction.

\begin{table}[h]
\centering\footnotesize
\renewcommand{\arraystretch}{0.96}
\begin{tabular}{l cc c cc}
\toprule
& \multicolumn{2}{c}{$\Vm_M$ (\%)} & $\Om$ & \multicolumn{2}{c}{$\Gamma_{\mathrm{cov}}$}\\
\cmidrule(lr){2-3}\cmidrule(lr){4-4}\cmidrule(lr){5-6}
Dataset & $x$ & $\tilde{x}$ & $x{=}\tilde{x}$ & $x$ & $\tilde{x}$\\
\midrule
ECL     & $+33.0$ & $-35.0$ & 0.639 & 0.000 & 0.000\\
Traffic & $+33.8$ & $+0.7$ & 0.571 & 0.129 & 0.017\\
\addlinespace[2pt]
ETTm1   & $-14.5$ & $-25.3$ & 0.541 & 0.000 & 0.000\\
ETTh1   & $-18.2$ & $-29.2$ & 0.522 & 0.000 & 0.000\\
Weather & $-35.0$ & $-43.7$ & 0.603 & 0.000 & 0.000\\
ETTh2   & $-40.9$ & $-34.8$ & 0.671 & 0.000 & 0.000\\
ETTm2   & $-47.4$ & $-38.7$ & 0.646 & 0.000 & 0.000\\
\bottomrule
\end{tabular}
\caption{\textbf{Full E1 spectrum-controlled table} (all seven benchmarks). Only ECL and Traffic carry positive retrieval value on $x$; $\Om$ is frozen throughout and $\Gamma_{\mathrm{cov}}$ is zero everywhere except Traffic.}
\label{tab:gapfull}
\end{table}

\subsection{Out-of-distribution transfer table (E5)}
Rules frozen on the seven development benchmarks, evaluated on three withheld
datasets (Table~\ref{tab:oodfull}, all $S$ pooled); each cell is the fraction of channels whose
context-value sign the rule calls correctly. Subheads give cell count and the
fraction with positive retrieval value. These sets carry almost no beyond-spectrum
value, so ``predict no help'' scores high by default (SCP's $96.4\%$ on ILI is
exactly this); the discriminating comparison is against the in-distribution
\emph{prior}.

\begin{table}[h]
\centering\footnotesize
\renewcommand{\arraystretch}{0.96}
\begin{tabular}{l ccc}
\toprule
Frozen rule & Exchange & ILI & M5\\
 & {\scriptsize $n{=}32$, $0\%{+}$} & {\scriptsize $n{=}28$, $4\%{+}$} & {\scriptsize $n{=}400$, $22\%{+}$}\\
\midrule
$\Delta_{\mathrm{nl}}$ (ours) & \textbf{100} & 28.6 & \textbf{50.7}\\
$\Gamma_{\mathrm{cov}}$ (ours) & \textbf{100} & 21.4 & 48.2\\
$\Omega$        & \textbf{100} & 3.6 & 22.0\\
SCP             & \textbf{100} & \textbf{96.4} & 22.0\\
\addlinespace[2pt]
prior           & 0.0 & 3.6 & 22.0\\
\bottomrule
\end{tabular}
\caption{\textbf{Full E5 out-of-distribution transfer.} Frozen rules on three withheld sets; the prior row carries the development majority.}
\label{tab:oodfull}
\end{table}

\end{document}